\documentclass[sigconf]{acmart}
\usepackage{balance}
\usepackage{algorithm}
\usepackage{algorithmicx}
\usepackage{graphicx}
\usepackage{stfloats}
\usepackage{subcaption}
\usepackage{algpseudocode}
\usepackage{amsthm}

\DeclareMathOperator*{\argmax}{arg\,max}
\DeclareMathOperator*{\argmin}{arg\,min}

\newtheorem{asu}{\textbf{Assumption}}
\newcounter{subassumption}[asu]

\makeatletter
\renewcommand{\p@subassumption}{\theasu}% Counter prefix.
\makeatother

\newtheorem{prop}{\textbf{Proposition}}
\newcounter{subproposition}[prop]

\makeatletter
\renewcommand{\p@subproposition}{\theprop}% Counter prefix.
\makeatother

 \newtheorem{defi}{\textbf{Defination}}
\newcounter{subdefination}[defi]

\makeatletter
\renewcommand{\p@subdefination}{\thedefi}% Counter prefix.
\makeatother

\algnewcommand{\LeftComment}[1]{\State \(\triangleright\) #1}
\def\BibTeX{{\rm B\kern-.05em{\sc i\kern-.025em b}\kern-.08emT\kern-.1667em\lower.7ex\hbox{E}\kern-.125emX}}
    
\copyrightyear{2019}
\acmYear{2019}
\setcopyright{acmcopyright}
\acmConference[CIKM '19] {The 28th ACM International Conference on Information and Knowledge Management}{November 3--7, 2019}{Beijing, China}
\acmPrice{15.00}
\acmDOI{10.1145/3357384.3358031}
\acmISBN{978-1-4503-6976-3/19/11}

\begin{document}

\fancyhead{}
  % do not delete this code.

% The "title" command has an optional parameter, allowing the author to define a "short title" to be used in page headers.
%\title{Sequential Cost-Effective Incentive Allocation using Constrained Markov Decision Processes}
\title{Model-based Constrained MDP for Budget Allocation in Sequential Incentive Marketing}

% The "author" command and its associated commands are used to define the authors and their affiliations.
% Of note is the shared affiliation of the first two authors, and the "authornote" and "authornotemark" commands
% used to denote shared contribution to the research.
\author{Shuai Xiao}
\authornote{Authors contributed equally to this research.}
\affiliation{%
  \institution{Ant Financial Services Group}
  \city{Shanghai}
  \country{China}}
\email{shuai.xsh@antfin.com}

\author{Le Guo}
\authornotemark[1]
\affiliation{%
  \institution{Ant Financial Services Group}
  \city{Beijing}
  \country{China}}
\email{guole.gl@antfin.com}

\author{Zaifan Jiang}
\authornotemark[1]
\affiliation{%
  \institution{Ant Financial Services Group}
  \city{Beijing}
  \country{China}}
\email{zaifan.jzf@antfin.com}

\author{Lei Lv}
\affiliation{%
  \institution{Ant Financial Services Group}
  \city{Beijing}
  \country{China}}
\email{lvlei.ll@antfin.com}
 
\author{Yuanbo Chen}
\affiliation{%
  \institution{Ant Financial Services Group}
  \city{Beijing}
  \country{China}}
\email{yuanbo.cyb@antfin.com}

\author{Jun Zhu}
\affiliation{%
  \institution{Ant Financial Services Group}
  \city{Beijing}
  \country{China}}
\email{elizhu.zj@antfin.com}

\author{Shuang Yang}
\authornote{Corresponding author.}
\affiliation{%
  \institution{Ant Financial Services Group}
  \city{San Mateo}
  \state{CA}
  \country{USA}}
\email{shuang.yang@antfin.com}

%
% By default, the full list of authors will be used in the page headers. Often, this list is too long, and will overlap
% other information printed in the page headers. This command allows the author to define a more concise list
% of authors' names for this purpose.
\renewcommand{\shortauthors}{Shuai Xiao, et al.}

%
% The abstract is a short summary of the work to be presented in the article.
\begin{abstract}
Sequential incentive marketing is an important approach for online businesses to acquire customers, increase loyalty and boost sales. How to effectively allocate the incentives so as to maximize the return (e.g., business objectives) under the budget constraint, however, is less studied in the literature. This problem is technically challenging due to the facts that 1) the allocation strategy has to be learned using historically logged data, which is counterfactual in nature, and 2) both the optimality and feasibility (i.e., that cost cannot exceed budget) needs to be assessed before being deployed to online systems. In this paper, we formulate the problem as a constrained Markov decision process (CMDP). To solve the CMDP problem with logged counterfactual data, we propose an efficient learning algorithm which combines bisection search and model-based planning. First, the CMDP is converted into its dual using Lagrangian relaxation, which is proved to be monotonic with respect to the dual variable. Furthermore, we show that the dual problem can be solved by policy learning, with the optimal dual variable being found efficiently via bisection search (i.e., by taking advantage of the  monotonicity). Lastly, we show that model-based planing can be used to effectively accelerate the joint optimization process without retraining the policy for every dual variable. Empirical results on synthetic and real marketing datasets confirm the effectiveness of our methods.

\end{abstract}



%
% Keywords. The author(s) should pick words that accurately describe the work being
% presented. Separate the keywords with commas.

%
% A "teaser" image appears between the author and affiliation information and the body 
% of the document, and typically spans the page. 

% \includegraphics[width=\textwidth]{sampleteaser}
% \Description{Enjoying the baseball game from the third-base seats. Ichiro Suzuki preparing to bat.}

%
% This command processes the author and affiliation and title information and builds
% the first part of the formatted document.
\maketitle

\section{Introduction}
Marketing with a form of incentives such as monetary prizes is a common approach especially in today's online internet industry. For example, in a typical online promotion, the owner of the campaign offers prizes such as coupons to encourage its target customers for certain favorable actions such as clicks or conversions. This campaign can be run only once such as in online advertising, or it can be run repetitively for multiple times throughout the lifecycle of a customer. The latter is becoming more important as companies are increasingly seeing the opportunities not only to acquire customers but also to increase their loyalty and boost sales. This problem is called sequential incentive marketing. 
Figure~\ref{fig:example} shows a real-world example where our problem emerges from. During one of market campaigning activities, Alipay repetitively sends red envelopes (coupons) with different amount of money to its users for a couple of days. Each red envelope incurs certain cost if consumed. The objective of the platform is to maximize the user engagement (e.g., the total times that people consume the red envelopes) under a global budget constraint through personalized sequential incentive allocation.

\begin{figure}[htb]
\centering
\includegraphics[width=.4\textwidth]{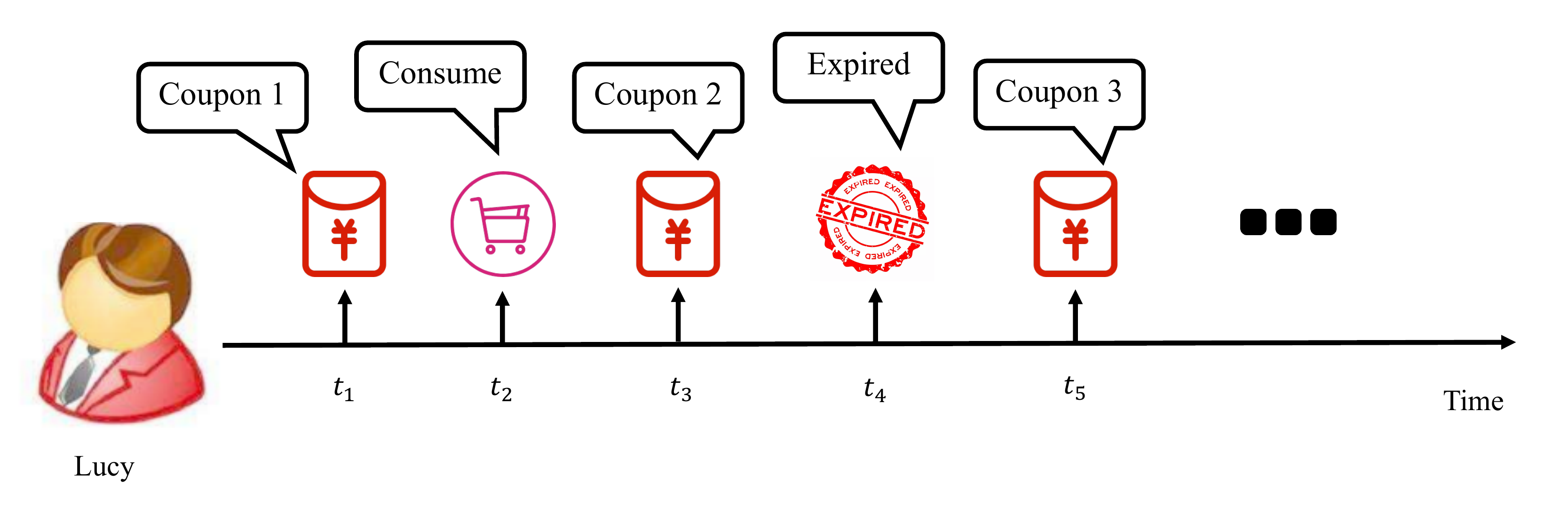}
\caption{ The illustrating example for constrained incentive allocation problem. During one of market campaigning activities, Alipay repetitively sends red envelopes with different amount of money to its users for a couple of days. Each red envelope incurs certain cost if consumed. The objective of the platform is to maximize the user engagement (e.g., the total times that users consume the red envelopes) under a global budget constraint through personalized sequential incentive allocation.}
\label{fig:example}
\end{figure}

Sequential incentive marketing poses unique technical challenges.  One of the difficulties of optimizing such problems is that only the feedback of the chosen recommendations (bandit feedback) is observed when multiple potential items for recommendation exist. Previous works~\cite{joachims2018deep, swaminathan2015batch, swaminathan2015counterfactual} have studied learning from logged bandit feedback with the help of counterfactual policy optimization without constraints. Another challenge is that both the optimality and feasibility (i.e., that cost cannot exceed budget) needs to be assessed before being deployed to online systems. In the industrial setting, the allocation strategy should be learned and verified in an off-policy manner from logged data because on-policy learning of such strategies has uncontrollable risks as the budget can't be reverted once dispensed. Therefore, vanilla on-policy algorithms are not suitable here since they rely on realtime interaction with the industrial environment to collect feedback of the current policy.
Furthermore, as the amount of data (i.e., billions of users and hundreds of items) is huge and the recommender system often include neural networks as modules, batch-training of such systems is usually necessary.

In this paper, we focus on sequential incentive recommendations with a global constraint, as is often the case in real-world industrial settings.
For this problem, users are repetitively recommended certain items from a candidate set. Each item incurs a cost if consumed by users. The objective is to maximize expected rewards such as CVR while the cost doesn't exceed the global budget. To solve this problem, we first formulate it as constrained Markov decision process (CMDP) where MDP describes the repetitive recommendation process for each user. 
This sequential incentive allocation problem with constraint has been less studied before.
Previous works~\cite{joachims2018deep, swaminathan2015batch, swaminathan2015counterfactual} solve contextual bandit problems from logged bandit feedback without constraints.  Lopez et al.~\cite{lopez2019cost} target at constrained bandit problem without considering the sequential allocation scenario where consecutive allocations have inter-dependence. Achiam et al.~\cite{achiam2017constrained} extend trust region optimization method~\cite{schulman2015trust} to solve CMDP for high-dimensional control, which falls into the category of on-policy reinforcement learning. The optimization relies on on-line data collection from the interactions between the policy to optimize and the environment, which is inapplicable in our real systems.

To solve the CMDP problem with logged off-line data, we propose an efficient learning algorithm which combines bisection search and model-based planning. Firstly, the primary CMDP is converted into its Lagrangian dual problem. The Lagrangian is formed by adding the budget constraint multiplying by a Lagrangian multiplier (also called dual variable) to the original objective function. We prove that the cost of incentive allocation decreases monotonically as the dual variable increases. Therefore the optimal dual variable for the dual problem can be identified efficiently through bisection search. In the learning process, the policy would have to be retrained for every value of the dual variable, which could be extremely time-consuming. To alleviate the heavy computational cost, model-based planning is also employed which enables one-pass policy training during the whole learning process. 
The transformation from primary CMDP to the dual problem also makes batch-training of CMDP possible, which is important as the amount of training data is very large and allocation systems often include neural networks as modules.

The primary contributions of our work are the following. 1). For real system, the allocation strategy has to be learned from logged off-line data and verified before applied to online system. We propose a novel formulation to sequential incentive allocations problem which allows for strategy learning and verifying from logged off-line data. This formulation also allows for batch-training which is vital for deep neural networks and large-scale datasets. 2). Efficient learning algorithm is devised for the proposed formulation based on theoretical findings. To accelerate to the learning process, bisection search is used based on the theoretical finding that the cost of incentive allocation is monotonic to the dual variable. 3). Model-based planning is used for policy updating so that the policy can be trained only once during the dual variable searching process.

%%%%%%%%%%%%%%%%%%%%%%%%%%%%%%%%%%%%%%%%
\section{Related Work}

\subsection{Batch Learning from Bandit Feedback}

In real recommendation system, only feedback of executed actions is observed. Learning from logged feedback data, also called bandit problem, naturally belongs to counterfactual inference. It's necessary to consider the counterfactual risk when evaluating models and minimizing counterfactual risk becomes a reasonable objective because of incomplete observations. Most methods employ importance sampling-based estimators to calculate the counterfactual risk of new policies. Alekk et al.~\cite{agarwal2014taming} propose a online learning algorithm for contextual bandit problem through iterative collections of feedback from real systems. As the amount of data in industry is huge, deep neural networks are often embedded as estimation modules for recommendation system which renders the batch-learning of such systems necessary. Thorsten et al. introduce Counterfactual Risk Minimization(CRM) principle and propose a series of works~\cite{swaminathan2015counterfactual, swaminathan2015batch, swaminathan2015self} performing batch learning from logged bandit feedback. SNIPS~\cite{swaminathan2015self} is proposed to solve the propensity overfitting problem of CRM through self-normalization. Recently, they propose a deep learning based model called BanditNet~\cite{joachims2018deep}, and convert the objective into a constrained optimization problem to allow the training neural network on a large amount of bandit data using stochastic gradient descent (SGD) optimization. Lopez et al.~\cite{lopez2019cost} introduce structured reward and HSIC to alleviate data collection bias, and use binary search to find a deterministic policy which satisfies the budget constraint. All these methods don't consider sequential allocation setting and focus on contextual bandits problem where actions are independent. In this paper, we focus on sequential allocation problem where sequential actions are correlated.

\subsection{Counterfactual Policy Evaluation}
For industrial applications, evaluating new policies before applying to online systems is necessary to ensure the safety. For bandit problems or sequential decision problems where partial feedbacks are observed, counterfactual policy evaluation is usually employed to assess the expectation of newly-developed policies. Such evaluations are based importance-sampling where logged feedback data from an old policy $\pi_b$ is served as a proxy to evaluate a new policy $\pi$ \cite{precup2000eligibility} in the following equation:
\begin{align} \notag
E_{\pi} (f) = \int \pi *f(x) dx = \int \pi_{b} \frac{\pi}{\pi_{b}} *f(x) dx = E_{\pi_b{}} (\frac{\pi}{\pi_{b}}*f) 
\end{align}
where $f$ is the reward function.

To reduce the variance of vanilla importance sampling methods, Jiang et al.~\cite{jiang2016doubly} extend doubly robust (DR) estimator for contextual bandits~\cite{dudik2014doubly} to reinforcement learning. They combine value estimation with importance sampling to alleviate the problem of high variance. Thomas et al.~\cite{thomas2016data} propose two methods to reduce the variance of DR at the cost of introducing a bias. Mehrdad et al.~\cite{farajtabar2018more} reduce the evaluation variance further by directly minimizing the variance of the doubly robust estimator.

\subsection{Constrained Policy Optimization}
Sequential allocation problems with constraints are mostly formulated as the constrained Markov Decision Process(CMDP)~\cite{altman1999constrained}. Optimal policies for finite CMDP problem with known dynamics and finite states can be solved by linear programming. However, learning methods for high-dimensional control are lacking~\cite{achiam2017constrained}. 
Achiam et al.~\cite{achiam2017constrained} extend trust region optimization methods~\cite{schulman2015trust} to solve CMDP for continuous actions. Those methods rely on on-policy data collection from environment which is inapplicable in industrial settings as the policy isn't allowed to explore and learn from scratch in real system due to the unbearable cost. Di et al.~ \cite{wu2018budget} formulate budget constrained bidding as a CMDP problem and the Lagrangian multiplier are actions. They treat the constraint as parts of the environment, and search optimal Lagrangian multiplier sequences from logged data where lots of samples exist that the budget is completely consumed. This method is not applicable in the case where logged data doesn't have experience that the budget are consumed completely.

%%%%%%%%%%%%%%%%%%%%%%%%%%%%%%%%%%%%%%%%
\section{Dual Method for CMDP}

\subsection{CMDP Formulation}
We formulate the sequential cost-effective incentive allocation problem with budget constaints as the Constrained Markov Decision Process (CMDP), which can be represented as a ($S, A, P, R, C, \mu, \gamma$, b) tuple:
\begin{itemize}
\item $\mathbf{S}$: The state space describing the user context, such as the user features.
\item $\mathbf{A}$: The action space containing candidate items for allocation.
\item $\mathbf{P}$: A probability distribution: $S * A \to S$, describing the dynamic transition from current state to the next one after taking action a.
\item $\mathbf{R}$: $S \times A \to \mathbb{R}$ is the reward function which maps states and actions to certain real number.
\item $\mathbf{C}$: $S \times A \to \mathbb{R}$  is the cost function corresponding to the budget consumed.
\item $\mathbf{\mu}$ is the initial state distribution for state : $S_0$.
\item $\mathbf{\gamma}  \in [0,1]$ is the discount factor for future rewards, which means how important intermediate rewards are. If $\mathbf{\gamma} =1$, then rewards are all equally important.
\item $\mathbf{b}$ is the global budget constraint, which the cost can't exceed.
\end{itemize}

To learn the CMDP problem,  logged feedback data is collected from an old behavior policy $\pi_b$ that interacted with the real system in the past. We assume $\pi_b$ is stationary for simplicity. The logged data $D$ are lists of tuples with six elements, consisting of observed state $s_i$ , action $a_i \sim \pi_b(* | s_i)$, the propensity  $p_i$ defined as $\pi_b(a_i | s_i)$,  the observed reward $r_i$,  the cost $c_i$ and the next state $s_{i+1}$.

\begin{equation}\notag
D=[(s_0, a_0, p_0, r_0, c_0, s_0), ..., (s_n, a_n, p_n, r_n, c_n, s_{n+1})]
\end{equation}

The goal of learning the CMDP problem is to find a policy $\pi(a|s)$ from $D$ that maximizes an objective function, $J(\pi)$, which is usually a cumulative discounted reward, $J(\pi) = E_{\tau \sim \pi} [ \sum_{t=0}^T \gamma^t R(s_t, a_t)]$ while the cost $J_{C}(\pi)$ doesn't exceed the budget constraint $b$ where $J_{C}(\pi) = E_{\tau \sim \pi} [ \sum_{t=0}^T \gamma^t C(s_t, a_t)]$. Here $\tau \sim \pi$  is shorthand for indicating that the distribution over trajectories depends on $\pi$: $s_0 \sim \mu, a_t \sim \pi, s_{t+1} \sim P$. To summarize, the CMDP can be formulated as the following equation:

\begin{equation}\label{origCMDP}
\begin{aligned}
p^* := \argmax_\pi &\quad J(\pi) \\
s.t.&\quad  J_C(\pi) \le b
\end{aligned}
\end{equation}

or equivalent equation:

\begin{equation}\label{CMDP}
\begin{aligned}
p^* :=  \argmin_\pi &\quad -J(\pi) \\
s.t.&\quad  J_C(\pi) \le b
\end{aligned}
\end{equation}

\subsection{Lagrangian Dual problem for CMDP}
To efficiently solve the CMDP problem in Equation~\ref{CMDP}, we firstly convert the primary problem into a Lagrangian dual problem by adding the constraint term to the original objective function to form the Lagrangian of CMDP as follows:
\begin{equation}\label{Lagrangian} \notag
L(\pi, \lambda) = -J(\pi) + \lambda (J_C(\pi) - b) 
\end{equation}
where $\lambda$ is the Lagrangian multiplier.

Then the Lagrangian dual problem of CMDP can be formulated as the follows:
\begin{equation}\label{dual problem}
\begin{aligned}
d^* = \max_\lambda&\quad g(\lambda)\\
s.t.&\quad \lambda \ge 0
\end{aligned}
\end{equation}
where $g(\lambda) = \min_{\pi} -J(\pi) + \lambda (J_C(\pi) - b)$ is called Lagrangian dual function. Here we refer $\pi$ as the primary variable and $\lambda$ as the dual variable.  For any $\lambda\ge 0$ and feasible $\pi$, $p^* \ge g(\lambda)$ always holds which means that  $g(\lambda)$ is a lower bound of $p^*$.

In the following, we first prove that the solution of the Lagrangian dual problem of CMDP in Equation~\ref{dual problem} exists.
% and is equivalent to that of its original problem in Equation~\ref{CMDP}. 
Then the learning algorithm of the Lagrangian dual problem of CMDP is given.
%and equivalence of solutions between primary and dual problem
Before proceeding into the proof of existence of the Lagrangian dual problem, we need to give basic definitions and assumptions that our proof relies on.
\begin{defi}\label{lowbound}
$\pi_l(s) = \argmax_a C(s, a)$: The lowest-cost policy always chooses actions with minimum cost at every state. $\pi_h(s) = \argmax_a C(s, a)$: The highest-cost policy always chooses actions with maximum cost at every state.
\end{defi}

\begin{asu}\label{strict feasibility}
Budget constraint under the lowest-cost policy $\pi_l$ is strictly feasible: $J_C(\pi_l) < b$. 
\end{asu}

\begin{asu}\label{violate policy}
The cost exceeds the budget constraint under the highest-cost policy $\pi_h$: $J_C(\pi_h) > b$.
\end{asu}

\begin{prop} [Existence of Solutions] \label{positive lambda}
When constraint satisfies assumption~\ref{strict feasibility} and~\ref{violate policy}, then the optimal value for dual variable $\lambda$ in Equation~\ref{dual problem} always exists and is positive, $\lambda^* > 0$.
\end{prop}

\begin{proof}
Under the assumption~\ref{strict feasibility}, a policy that doesn't exceed the budget exists and therefore the primary and dual problem always have a strictly feasible solution. Under the assumption~\ref{violate policy}, it's obvious that there always exists a positive optimal Lagrangian multiplier $\lambda$ which penalizes the budget constraint violation. If the cost of high-cost policy doesn't exceed the budget, then the CMDP problem degrades into MDP without constraints and the optimal dual variable is always zero. In conclusion, based on the two natural assumptions which hold in practical problems, we have the above proposition.
\end{proof}

\subsection{Solving the Dual Problem}
After proving the existence of solutions of the primary and dual problem, we iteratively improve the lower bound by optimizing the dual problem instead.
To solve the dual problem in equation~\ref{dual problem}, dual ascent method~\cite{boyd2004convex} can be used, where primary variable and dual variable can be optimized alternatively via gradient ascent. In practice, the computational cost of dual ascent method is extremely high as for each dual variable $\lambda$, we need to learn the optimal policy $\pi$ and for each $\pi$, we need to run counterfactual policy evaluation (CPE) to compute the sub-gradient of $\lambda$.
To reduce the heavy computational cost, we develop an novel algorithm which can efficiently identify optimal dual variable $\lambda$ based on theoretical deductions. We first derive that the cost of incentive allocation is monotonic to the Lagrangian multiplier $\lambda$. Then a faster bisection method is proposed to determine the optimal dual variable $\lambda^*$.

\begin{theorem} \label{monotonic budget}
Let Lagrangian 
$
L(\pi, \lambda) = -J(\pi) + \lambda(J_C(\pi) - b)
$, and 
$
\pi_\lambda=\argmax_{\pi} L(\pi, \lambda)
$.
 If $\lambda_a > \lambda_b$, then $J_C(\pi_{\lambda_a}) \leq J_C(\pi_{\lambda_b})$. That's to say, $J_C(\pi)$ is monotonic with $\lambda$.
\end{theorem}

\begin{proof}
Because $\pi_{\lambda_a}$ and $\pi_{\lambda_b}$ is the minimizer of Lagrangian dual function, therefore
\begin{equation} \notag
-J(\pi_{\lambda_a}) + \lambda_a(J_C(\pi_{\lambda_a})-b) \leq -J(\pi_{\lambda_b})+\lambda_a(J_C(\pi_{\lambda_b})-b)
\end{equation}
\begin{equation} \notag
-J(\pi_{\lambda_b}) + \lambda_b(J_C(\pi_{\lambda_b})-b) \leq -J(\pi_{\lambda_a})+\lambda_b(J_C(\pi_{\lambda_a})-b)
\end{equation}
Adding the two inequalities
$$
\begin{aligned}
&-J(\pi_{\lambda_a}) - J(\pi_{\lambda_b}) + \lambda_a(J_C(\pi_{\lambda_a})-b)+\lambda_b(J_C(\pi_{\lambda_b})-b)\\
\leq & -J(\pi_{\lambda_b}) - J(\pi_{\lambda_a}) +\lambda_a(J_C(\pi_{\lambda_b})-b) + \lambda_b(J_C(\pi_{\lambda_a}) -b)\\
&\Downarrow\\
&(\lambda_a-\lambda_b)(J_C(\pi_{\lambda_a}) - J_C(\pi_{\lambda_b}))\leq 0\\
\end{aligned}
$$
because $\lambda_a > \lambda_b$, which leads to the following conclusion:
$$
\begin{aligned}
&J_C(\pi_{\lambda_a}) \leq J_C(\pi_{\lambda_b})
\end{aligned}
$$
Proof finishes.
\end{proof}

With the result that the cost of incentive allocation is monotonic with dual variable $\lambda$, the dual problem can be solved much faster than dual ascent method~\cite{boyd2004convex}. We firstly design a mechanism to identify optimal dual variable based on bi-section search, which empirically converges faster than dual ascent method, consequently reducing computational time greatly. The dual ascent usually has near-linear convergence rate while the bisection search can have exponential convergence rate. 
We firstly compute the lower bound $\lambda_{l}$ and upper bound $\lambda_{u}$ of dual variable $\lambda$. Then we learn the policy $\pi_{\lambda_m}$ from logged off-line data by fixing the dual variable $\lambda_m$ where  $\lambda_m = \frac{\lambda_{l} + \lambda_{u}}{2}$ is the middle point lying between lower and upper bound. Then the pseudo sub-gradient of $\lambda$ at the middle point $\lambda_m $ is obtained through running counterfactual policy evaluation (CPE) over logged data to compute $J_C - b$, which indicates the direction of next move for dual vaiable $\lambda$. In summary, the details of dual problem learning are given in Algorithm~\ref{dual method cmdp}.

\begin{algorithm}[h]
\caption{Dual Method Learning Framework for Constrained Markov Decision Processes}
\label{dual method cmdp}
\begin{algorithmic}[1]

\Require Initial policy $\pi_0$, logged training data $D_{train}$,  logged evaluation data $D_{val}$. 
\Require Lower bound $\lambda_l$ and upped bound $\lambda_u$ of dual variable $\lambda$.
\Require Policy learning algorithm $PL$, and Counterfactual policy evaluation algorithm $CPE$.

\While {$\lambda_{l} < \lambda_{u}$}
   \State Identify middle point of dual variable $\lambda_m =\frac{\lambda_{l} + \lambda_{u}}{2}$
   \State Learning policy $\pi_{\lambda_m}$ using algorithm $PL$ from logged data $D_{train}$ with fixed $\lambda_m$.
   \State Run $CPE$ to get counterfactual cost $J_C$ = $CPE$ ( $D_{val}$,  $\pi_{\lambda_m}$ ).
   \State Compute pseudo sub-gradient of $\lambda$ to determine the search direction.
    \If { $J_C \le b$ and $|J_C - b| < \delta$  }
	\State break
    \EndIf
    \If {$J_C < b$}
    	\State $\lambda_{l} = \lambda_m$
    \Else
     	\State $\lambda_{u} = \lambda_m$
    \EndIf
\EndWhile
\State \Return Optimal dual variable $\lambda_m$ and policy $\pi_{\lambda_m}$.
%\EndProcedure

\end{algorithmic}
\end{algorithm}

The policy learning $PG$ in Algorithm~\ref{dual method cmdp} will be introduced in detail in Section~\ref{sec:pg}. To circumvent retraining the policy for each dual variable $\lambda$, a model-based planning algorithm is proposed to accelerate the learning process in Section~\ref{sec:model-based}.
To make the main idea of the paper coherent, we move the details of counterfactual policy evaluation (CPE) and identification of upper bound $\lambda_u$ of dual variable into Appendix~\ref{sec:CPE} and~\ref{proof lambda upper bound}.

%%%%%%%%%%%%%%%%%%%%%%%%%%%%%%%%%%%%%%%%
\section{Policy Learning with Dual Variable} 
For dual method learning in Algorithm~\ref{dual method cmdp}, we need to update policy for each $\lambda$. In this section, we firstly derive a variant of DQN as policy learning methods and then add entropy regularizer to improve both exploration and robustness in challenging decision-making tasks~\cite{geibel2005risk}. 
For the above approach, the policy has to be re-trained for each each $\lambda$, which is very time-consuming. To alleviate this problem, a model-based approach is firstly proposed to accelerate the training process. We'll give the details of policy learning and model-based acceleration in the following.

%%%%%%%%%%%%%%%%%%%%%%%%%%%%%%%%%%%%%%%%
\subsection{Policy Learning} \label{sec:pg}
As in our case that the cost is incurred for each action taken which means that the cost and reward have the same distribution depending on policy $\pi$, so the cost can be subsumed into the reward as shown in the following equation.
\begin{equation}\label{new reward} 
\begin{aligned}
L(\pi, \lambda)&=-J(\pi) + \lambda (J_C(\pi))\\
&= -\mathbb{E}_{\tau \sim \pi(\tau)}[r(\tau)] + \lambda \big( \mathbb{E}_{\tau\sim \pi(\tau)}[c(\tau)] - b \big)\\
&=-\mathbb{E}_{\tau\sim\pi(\tau)} \Bigg[ \sum_{i=1}^T(\underbrace{r(s_t,a_t)-\lambda c(s_t,a_t)}_{\text{new reward}}) \Bigg] -\lambda b
\end{aligned}
\end{equation}

As a result, we can use the new reward  $rc(s, a)=r(s,a) - \lambda c(s,a)$ to learn a DQN policy~\cite{mnih2015human, wang2015dueling} given $\lambda$. One benefit of this reformulation is that the new reward has reasonable interpretation that the policy taking actions with high rewards and low costs is desired. Another benefit is that this reformulation makes the batch-learning of DQN easier as the cost term disappears so we can sampling a batch of samples $\{(s_i, a_i, p_i, r_i, c_i, s_{i+1})\}$ for training as done in traditional DQN.

%%%%%%%%%%%%%%%%%%%%%%%%%%%%%%%%%%%%%%%%

To improve both exploration and robustness in challenging decision-making tasks~\cite{geibel2005risk}, we can also add entropy regularizer to the above reward-reshaped DQN. The entropy regularizer makes the policy to take actions as diverse as possible. In this way, we get a more stochastic DQN, which falls into the category of Soft-Q based algorithms~\cite{haarnoja2017reinforcement, schulman2017equivalence}. The objective function after adding an entropy regularizer to the policy becomes:
\begin{equation}\label{soft q}
\begin{aligned}
J_{\rho}(\pi)&\quad:= \mathbb{E}_{\tau\sim\pi(\tau)}[rc(\tau) - \rho g(\tau)]\\
\text{where}&\quad rc(\tau)=r(\tau) - \lambda c(\tau)\\
&\quad g(\tau) = \sum_{i=0}^T\gamma^i\log\pi(a_i|s_i)
\end{aligned}
\end{equation}

The optimal policy for entropy-regularized policy in equation~\ref{soft q} can be derived in a similar way to PCL~\cite{nachum2017bridging} and is given directly as bellow:
\begin{equation}\label{equation:soft-q} 
\begin{aligned}
Q^*(s,a)&=rc(s,a) + \gamma V^*(s')\\
V^*(s)&=\rho\log\sum_a\exp\{\frac{Q^*(s,a)}{\rho}\}\\
\pi^*(a|s)&=\exp\{\frac{Q^*(s,a) - V^*(s)}{\rho}\}\\
\end{aligned}
\end{equation}
From above equation, if $\rho$ is close to zero, then the stochastic DQN degrades into original DQN without entropy regularizer, where the action with the largest Q value is taken. If $\rho$ is infinitely large, then all actions have the same probability to be taken. The details of policy learning are summarized in Algorithm~\ref{update policy q-learning}.

\begin{algorithm}[h]
\caption{Policy Learning with Dual Variable}
\label{update policy q-learning}
\begin{algorithmic}[1]
\Require Dual variable $\lambda$, logged training data $D_{train}$. 

\State Calculate reshaped reward $rc(s, a)=r(s,a) - \lambda c(s,a)$.
\State Policy learning using batch SGD with loss~\ref{new reward} or ~\ref{soft q}.
\State Prob(s,a) = $\pi_{\lambda} (s,a) $ during inference. \Comment{Only used in model inference}
\State \Return Optimal policy $\pi_{\lambda}$

\end{algorithmic}
\end{algorithm}

%%%%%%%%%%%%%%%%%%%%%%%%%%%%%%%%%%%%%%%%
\subsection{Model-based Acceleration} \label{sec:model-based}
Even using bi-section search for $\lambda$, it's very time-consuming for re-training policy for multiple times. To alleviate the re-training problem, we can turn to the idea of model predictive control (MPC). If we have the dynamic model of the environment, we can use tree search based planning algorithm by expanding the full search tree from the current state to derive the optimal policy. To this end, we first learn the environment model from logged data and then derive optimal policy through tree-search based model planning.

Firstly, the cost and reward model, and state transition model are learned from logged data:
\begin{equation}\label{dynamic model} \notag
\begin{aligned}
\text{Transition model:}\quad\quad& T(s_t, a_t)  \rightarrow s_{t+1}\\
\text{Reward model:}\quad\quad& R(s_t, a_t) \rightarrow r_{t}\\
\text{Cost model:}\quad\quad& C(s_t, a_t) \rightarrow c_{t}
\end{aligned}
\end{equation}

Lots of ways exist to model transition dynamics, like neural network~\cite{nagabandi2018neural}, gaussian processes~\cite{deisenroth2011pilco} and latent dynamics ~\cite{watter2015embed, hafner2018learning, ha2018world}. Here we use a feed-forward deep neural network with a regression loss to learn the dynamics. The reward model,  typically a CTR prediction problem, can be learned with logistic regression~\cite{mcmahan2013ad}, logistic regression with GBDT~\cite{he2014practical} or Wide\&Deep model~\cite{cheng2016wide}.  In our paper, Wide\&Deep model is used as its flexibility suits our high-dimensional user feature. The cost model is learned with Wide\&Deep model with a regression loss.

With the learned model, we can perform model predictive control (MPC) to obtain policy $\pi$. In this work, we simply traverse the search tree over a finite horizon $H$ to calculate the Q values using the learned dynamics. Note that for higher dimensional action spaces and longer horizons, Monte Carlo sampling may be insufficient, which we leave for future work. The overall procedure of model-based approach is presented in Algorithm~\ref{update policy model-based planning}.

\begin{algorithm}[h]
\caption{Policy Learning with Dual Variable using Model-based Planning Acceleration}
\label{update policy model-based planning}
\begin{algorithmic}[1]

\Require Dual variable $\lambda$, entropy regularizer coefficient $\rho$, planning horizon $h$.
\Require Transition model $T(s_t, a_t)$, reward model $R(s_t, a_t)$, cost model $C(s_t, a_t)$.

\State Set $Q(s, a) = 0$ for any s, a.
\For { i = 1 ... N}
        \State Traverse local tree with depth $h$ from state $s$ using learned dynamic $T$.
	\State Calculate $Q(s, a), V(s)$ using equation~\ref{equation:soft-q}.
	\State Calculate the policy $\pi^*(a|s) =\exp\{\frac{Q^*(s,a) - V^*(s)}{\rho}\} $
\EndFor
\State
\Return Policy $\pi^*$

\end{algorithmic}
\end{algorithm}

This combination of learned model dynamics model and model predictive planning is beneficial in that the model is trained only once. By simply changing the $\lambda$, we can accomplish a variety of policy validation effectively, without a need for $\lambda$-specific retraining.

\renewcommand{\thefigure}{\arabic{figure}}
\renewcommand{\thesubfigure}{\alph{subfigure}}

\begin{figure*}[htb]
\centering
\includegraphics[width=.8\linewidth]{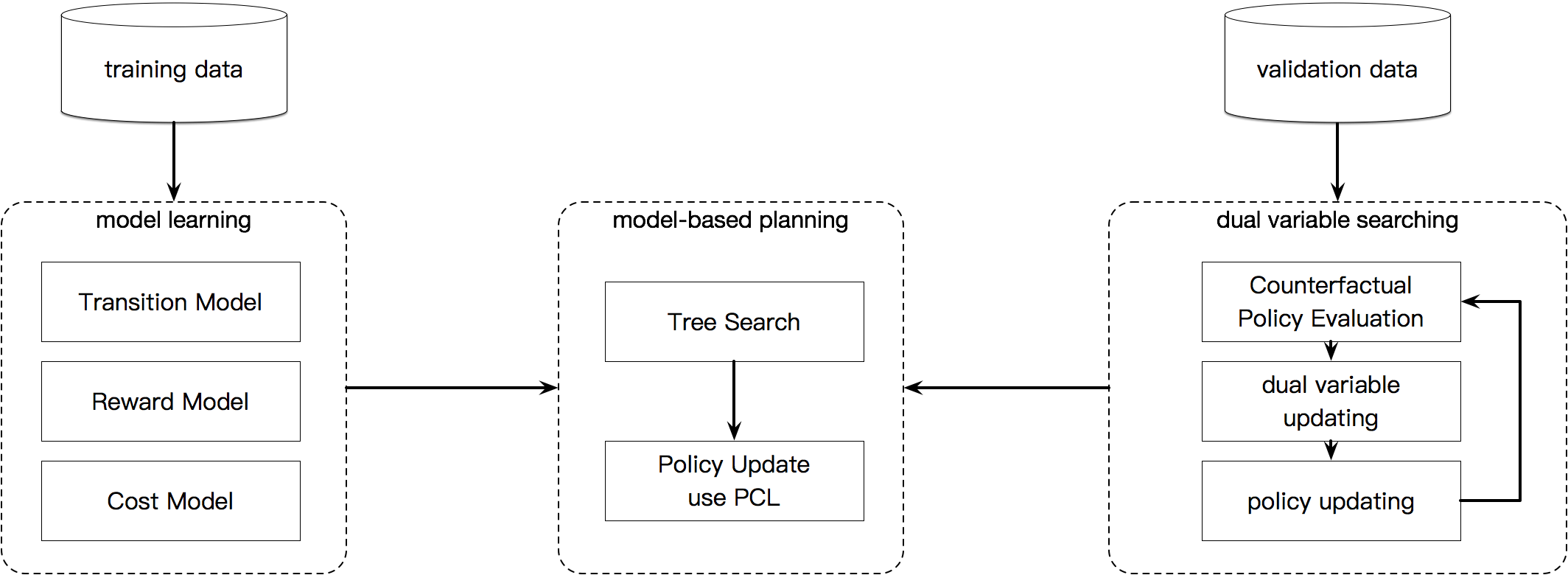}
\caption{The workflow of the proposed model-based acceleration. The main process shown in the right column is a closed loop for using dual ascent to solve Lagrangian dual problem for constrained MDP which contains dual variable identification and policy learning. The sub-gradient of the Lagrangian multiplier (dual variable) is calculated using counterfactual policy evaluation. The middle column corresponds to model-based acceleration and the right column shows the data flow for learning environment model.}
\label{fig:framework}
\end{figure*}

The overall workflow of the model-based acceleration is shown the Figure~\ref{fig:framework}. The main process shown in the right column is a closed loop for using dual ascent to solve Lagrangian dual problem for constrained MDP which contains dual variable identification and policy learning. The sub-gradient of the Lagrangian multiplier (dual variable) is calculated using counterfactual policy evaluation. The middle column corresponds to model-based acceleration and the right column shows the data flow for learning environment model.

%%%%%%%%%%%%%%%%%%%%%%%%%%%%%%%%%%%%%%%%
\section{Experimental results}
We conduct experiments on synthetic and real-world data. To testify the soundness of the proposed method, synthetic data which satisfies the three assumptions is simulated. For real-world data, its characteristics are firstly analyzed to check the validation of the assumptions and then the performances of the proposed method are demonstrated.

\subsection{Baselines}
To show the benefits of the proposed approach, several alternatives are used to demonstrate the advantages and disadvantages of different approaches, which are briefly introduced as follows:
\begin{itemize}
\item{Random:} Randomness is a simple and effective strategy to allocate incentives to customers when recommendation system works with a cold start. The policy allocate all items to customers at an uniform probability.  
\item{Constrained contextual bandit:} Constrained contextual bandit constitutes one-step Constrained MDP where correlations between steps are not considered.
\item{Constrained MDP:} Constrained Markov decision process is solved with the proposed fast bisection search with and without model-based acceleration.
\end{itemize}

\subsection{Synthetic data}

\noindent\textbf{Dataset Description}  We generate a two-step sequential contextual bandit dataset and actions are selected with a uniform distribution. All states in the first step are 0, and next states are $i+1$ when taking action $i, i=0,...n$. The reward for state $s$ action $i$ is sampled from a gaussian distribution with mean reward $r[s, i]$ and variance $v$. The cost for state $s$ and action $i$ is sampled from a gaussian distribution with mean cost $c[s, i]$ and variance $0.1$. We set $r[s, i] = r[0, s] + (i - s) * \beta_r$, $c[s, i] = r[0, s] + (i - s) * \beta_c$. Therefore, rewards $r[s, 0] ... r[s, n]$ and costs $c[s, 0] ... c[s, n]$ are in ascending order.

\noindent\textbf{Result} For the model-based approach, the reward and cost model should be learned at first. The true average reward and the predicted reward for hold-out validation data from the reward model are shown in Figure~\ref{fig:reward_accuracy}. The reward model can predict the rewards for different users accurately with an acceptable error. The true average cost and the predicted cost for hold-out validation data from the cost model are shown in Figure~\ref{fig:cost_accuracy}. The prediction accuracy for costs is very high. Note that there is a trade-off between speeding up of algorithm through model-based planning and the performance of the dual CMDP problem. If the prediction accuracy of model-based planning doesn't reach the acceptable criterion, we can roll back to use the original model without model-based acceleration.

\renewcommand{\thefigure}{\arabic{figure}}
\renewcommand{\thesubfigure}{\alph{subfigure}}

\begin{figure}[htb!]
\centering
\begin{subfigure}{.4\textwidth}
\includegraphics[width=.8\linewidth]{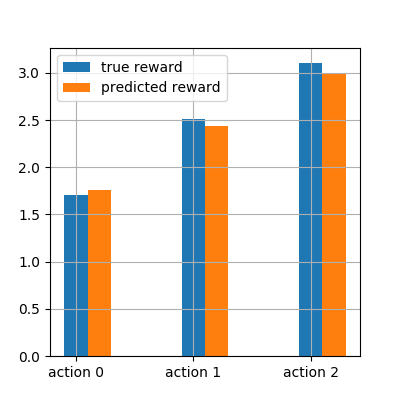}
\caption{The prediction accuracy for the reward model. Average ground-truth and predicted reward over the holdout validation dataset are shown.}
\label{fig:reward_accuracy}
\end{subfigure}
\begin{subfigure}{.4\textwidth}
\includegraphics[width=.8\linewidth]{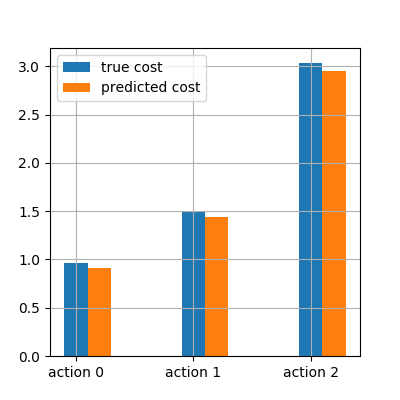}
\caption{The prediction accuracy for the cost model. Average ground-truth and predicted cost over the holdout validation dataset are shown.}
\label{fig:cost_accuracy}
\end{subfigure}
\caption{The prediction accuracy of the reward and cost model in the model-based approach.}
\end{figure}

For constrained MDP, we compare the performance and computational complexity for the proposed fast bisection search and model-based acceleration learning algorithms. The performances of the two learning algorithms are shown in Figure~\ref{model_free}. The two learning algorithms produce almost the same rewards under different budgets while the model-based acceleration doesn't need to re-train the policy for each $\lambda$.  The computational cost of bi-section search is $O(n)$ times as much as model-based acceleration where n is the search times for $\lambda$. Except for explicitly stated, the performance for constrained MDP is for the model-based acceleration in the later discussion.

\begin{figure}[htb]
\centering
\includegraphics[width=.4\textwidth]{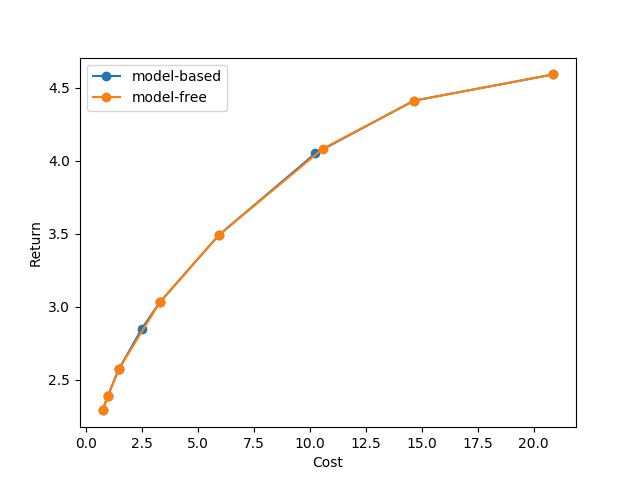}
\caption{ The performance of constrained MDP solved with the proposed fast bisection search and model-based acceleration learning algorithms. The "model-free" represents the bisection search and "model-based" represents model-based acceleration.}
\label{model_free}
\end{figure}

\begin{figure}[t!]
\centering
\begin{subfigure}[t]{.4\textwidth}
\includegraphics[width=.8\linewidth]{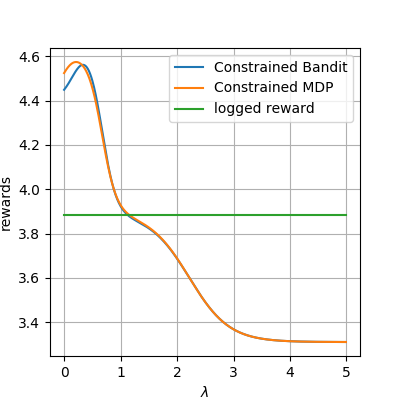}
\caption{Reward curves of different approaches regarding to $\lambda$. Rewards are computed over the holdout validation dataset using the models learned under different $\lambda$.}
\label{fig:reward_curve}
\end{subfigure}
\begin{subfigure}[t]{.4\textwidth}
\includegraphics[width=.8\linewidth]{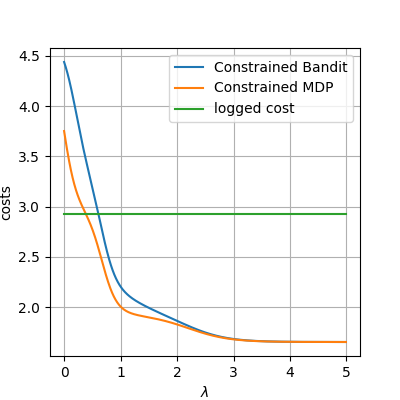}
\caption{Cost curves of different approaches regarding to $\lambda$. Costs are computed over the holdout validation dataset using the models learned under different $\lambda$.}
\label{fig:cost_curve}
\end{subfigure}
\caption{The performances of constrained contextual bandit and constrained MDP approaches.}
\label{fig:performance}
\end{figure}

The performances of constrained contextual bandit and constrained MDP approaches are shown in Figure~\ref{fig:performance}. The rewards and costs are computed over the holdout validation dataset using the models learned under different $\lambda$. The cost curve for the dual problem of constrained MDP in Figure~\ref{fig:cost_curve} decreases monotonically as the $\lambda$ increases. This fact empirically verifies the Theorem~\ref{monotonic budget}. The reward curves for the dual problem of constrained MDP in Figure~\ref{fig:reward_curve} almost decreases with $\lambda$. This is because that the costs decrease with $\lambda$, so the rewards decrease together with the costs as stated in Assumption~\ref{monotonicity}. As we can see, the reward curves for constrained bandit and MDP are almost overlapped. Therefore, for a fixed budget b, the $\lambda$ for constrained bandit is smaller than constrained MDP and the reward for constrained MDP is larger than that of constrained bandit. The Table~\ref{table:sim_result} shows the performances of different approaches under a fixed budget $b=2.9295$ on the synthetic dataset. The reward of constrained bandit has a 13.31\% increase over the random approach.  The proposed constrained MDP has an 16.11\% improvement over the random approach. 

\begin{table}[H]
\caption{Performances of different approaches under a fixed budget on the synthetic dataset.}
\begin{tabular}{|l |l |l|}
 \hline
Algorithm& cost &reward\\
\hline
\hline
 Random   &  2.9295    & 3.8836 \\
 Constrained contextual bandit&   2.9295  & 4.4006(\textbf{+13.31\%})\\
 Constrained MDP & 2.9295 & 4.5093(\textbf{+16.11\%})\\
 \hline
\end{tabular}
\label{table:sim_result}
\end{table}

\subsection{Real-world Dataset}

\noindent\textbf{Dataset Description} The Dataset comes from Alipay which is one of China's largest payment platforms owned by Ant Financial. During one of market campaigning activities, Alipay repetitively sends red envelopes to its users for a couple of days. The purpose is to improve users' engagement and activeness.  This real world data is used to testify our algorithm. In the dataset, each record contains user features, the amount of red envelope. We take the amount of red envelope as action and whether the user is active in the next day as reward. The amount of red envelop is also the cost incurred by that action. 
This application obeys the three assumptions. Especifically, for Assumption~\ref{monotonicity}, the more money in the red envelope, the higher probability users will be active in the next day.

\noindent\textbf{Results} 
Figure~\ref{action_reward} shows the accuracy of learning the model dynamics, which include the reward model and cost model. As shown, the predicting errors of reward and cost model are relevantly very small in our case. Therefore, the error introduced by model-based planning is small enough which won't deteriorate the performance of model-based acceleration algorithm too much.

\begin{figure}
\centering
	\begin{subfigure}{.4\textwidth}
		\centering
		\includegraphics[width=5cm]{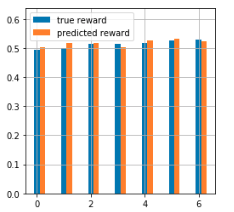}
		\caption{Prediction accuracy of reward}
	\end{subfigure}
	\begin{subfigure}{.4\textwidth}
		\centering
		\includegraphics[width=5cm]{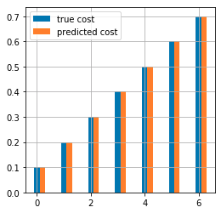}
		\caption{Prediction accuracy of cost}
	\end{subfigure}
	\caption{The prediction accuracy for the cost model. Average ground-truth and predicted cost over the holdout validation real dataset are shown.}
	\label{action_reward}
\end{figure}

Figure~\ref{lambda_reward} and Figure \ref{lambda_cost} describes relationships between reward(cost) and lambda. 
The reward decreases as $\lambda$ increases.
We can see cost decreases when $\lambda$ increases which means the Theorem~\ref{monotonic budget} hold for this real dataset.

\begin{figure}[htb]
\centering
\begin{subfigure}[t]{.4\textwidth}
\includegraphics[width=6cm]{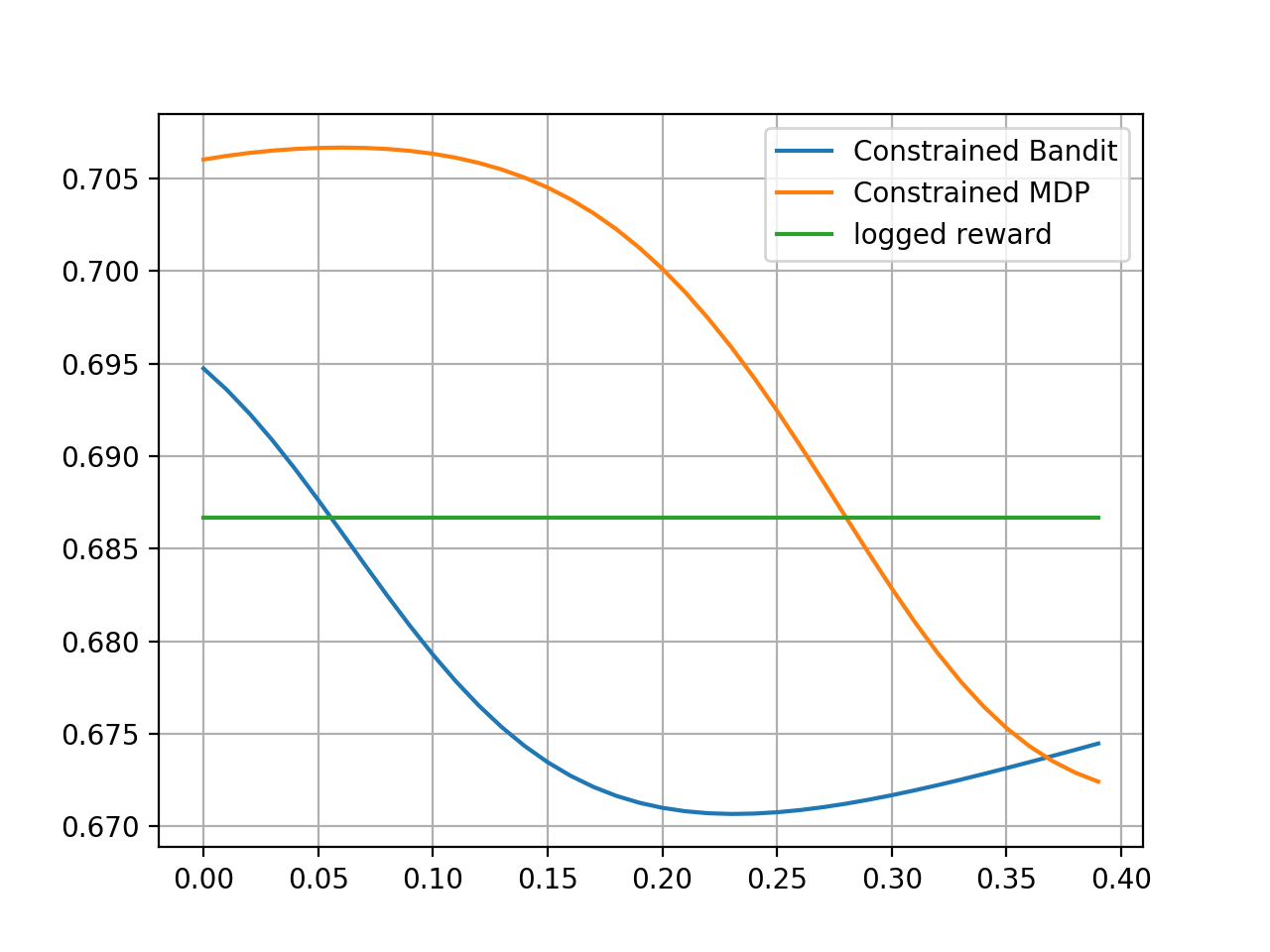}
\caption{Reward curves of constrained MDP approach regarding to $\lambda$. Rewards are computed over the holdout validation dataset using the models learned under different $\lambda$.}
\label{lambda_reward}
\end{subfigure}

\begin{subfigure}[t]{.4\textwidth}
\includegraphics[width=6cm]{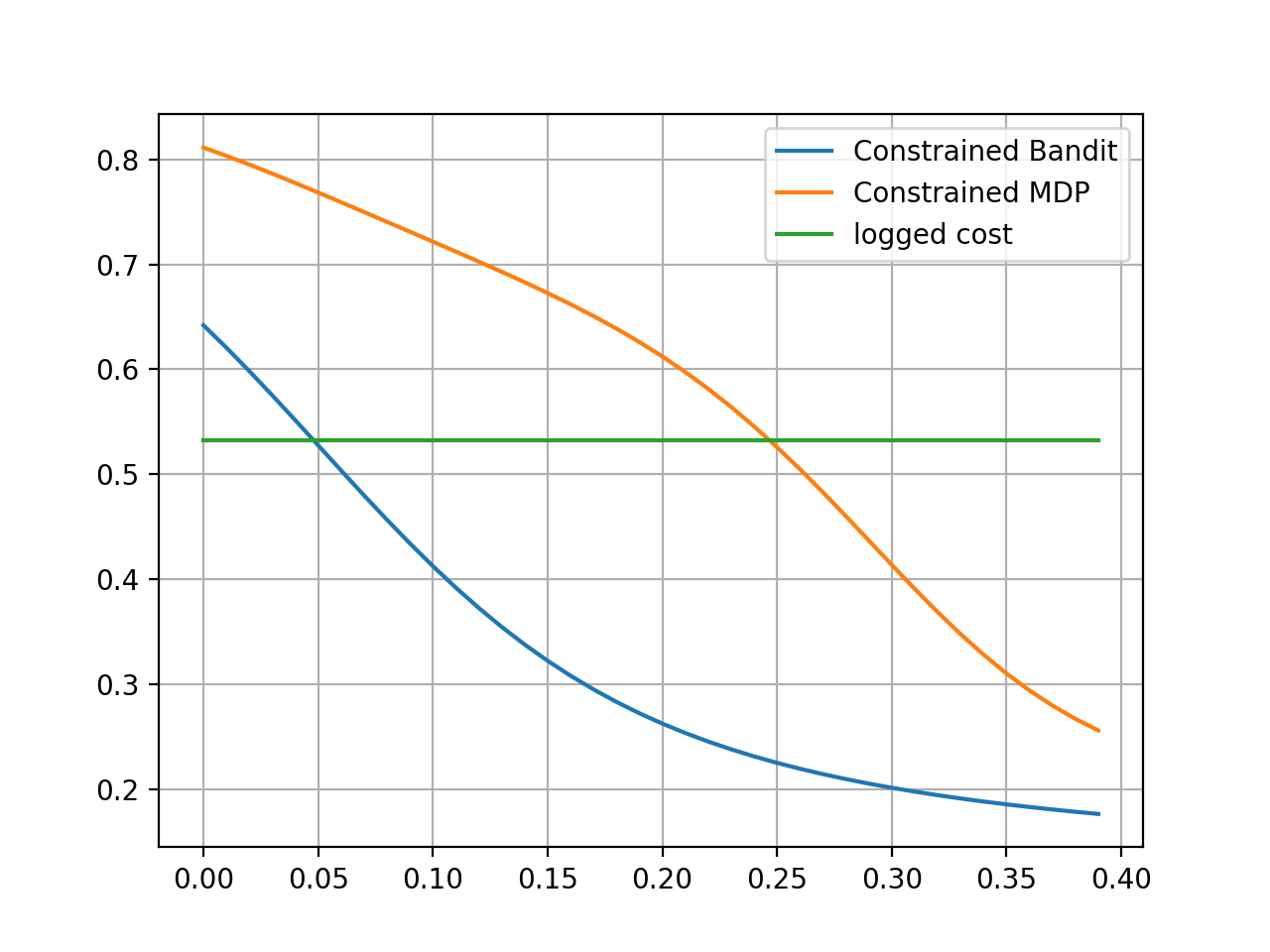}
\caption{Cost curve of constrained MDP approach regarding to $\lambda$. Costs are computed over the holdout validation dataset using the models learned under different $\lambda$.}
\label{lambda_cost}
\end{subfigure}

\caption{The performances of constrained MDP approach regarding to $\lambda$.}
\end{figure}

The Table~\ref{table:real_result} shows the performances of different approaches under a fixed budget $b=0.532$ on the real dataset. The reward of constrained bandit has a 8.97\% increase over the random approach.  The proposed constrained MDP outperforms alternatives over a notable margin, 13.61\% improvement over the random approach.  From another point of view, we evaluate the cost by fixing the campaigning reward. By setting the target reward to $0.687$, constrained MDP can achieve the goal with a cost 9.40\% lower than the random approach which constrained contextual bandit needs more budget.

\begin{table}[h!]
\caption{Performances of different approaches under a fixed budget on the real-world dataset.}
\begin{tabular}{|l |l |l|}
 \hline
Algorithm& cost &reward\\
\hline
\hline
 Random   &  0.532    & 0.687  \\
 \hline
 Fixed budget & &  \\
 \hline
 Constrained contextual bandit &   0.532  & 0.688(\textbf{+0.145\%})\\
 Constrained MDP & 0.532 & 0.693(\textbf{+0.873\%})\\
 \hline
  Fixed reward  &  & \\
  \hline
 Constrained contextual bandit &   0.503 (\textbf{-5.46\%}) & 0.687\\
 Constrained MDP& 0.482 (\textbf{-9.40\%}) & 0.687\\
 \hline
\end{tabular}
\label{table:real_result}
\end{table}

%%%%%%%%%%%%%%%%%%%%%%%%%%%%%%%%%%%%%%%%
\section{Conclusion}
This paper presents an efficient solution framework for sequential incentives allocation with the budget constraint which employs bisection search and model-based planning to solve the CMDP problem with logged counterfactual data. Empirical results on synthetic and real industrial data show its superior performances compared to alternatives.
For future work, we plan to explore ways to jointly optimize the primary variable $\pi$ and dual variable $\lambda$ as in Equation~\ref{dual problem}, which can converge faster to optimal solutions without model-based planning.
Another interesting direction is to use model-based approach to improve data-efficiency in reinforcement learning as collecting data from real systems is often costly or restricted due to practical concerns.

%
% The next two lines define the bibliography style to be used, and the bibliography file.
\bibliographystyle{ACM-Reference-Format}
\bibliography{sample-base}

% 
% If your work has an appendix, this is the place to put it.
\appendix

\section{Counterfactual Policy Evaluation (CPE)} \label{sec:CPE}
The details of Counterfactual Policy Evaluation is given in Algorithm~\ref{CPE}.
\begin{algorithm}[h]
\caption{Counterfactual Policy Evaluation}
\label{CPE}
\begin{algorithmic}[1]
\Require Policy $\pi$ to evaluate, logged evaluation data $D_{val}$. 
\State  Initial empty samples set $\Phi$.
\For {$(s, a, p, r, c) \in D_{val}$  }
\State $p_{new}$ = $\pi (s, a)$
\State Add transformed samples (($s,a,p,c,p_{new}$)) into samples set $\Phi$.
\EndFor
\State Evaluated cost $C$= Doubly Robust($\Phi$) \Comment{evaluate cost use doubly robust}.
\State \Return Evaluated cost $C$.

\end{algorithmic}
\end{algorithm}

\section{Upper bound for $\lambda$}\label{proof lambda upper bound}

When using bisection search,  we can use assumption~\ref{monotonicity} to calculate upper bound for dual variable $\lambda$.

\begin{asu}\label{monotonicity}
Monotonicity between reward and cost at every state,  when $c(s, a_0) > c(s, a_1)$, then $r(s, a_0) > r(s, a_1)$ 
\end{asu}
This assumption presents the simple idea that the larger reward you get, the higher cost you pay. This idea holds in most promotion and advertising scenarios.

The details of deriving upper bound are given as follows.

The basic idea to find an upper bound for $\lambda$ is the optimal policy for $\lambda$ is $\pi_l$ in definition~\ref{lowbound}. Assume the reward and cost is in ascending order for action $a_0, a_1, ...$,  define reward and action for action $a_i$ is $r_i, c_i$,  so $r_i < r_j, \forall i < j$ and $c_i < c_j, \forall i < j$, define upper bound of $\lambda$ as $\lambda_{u}$, for $\pi_l$, we need $\argmax_i r_i - \lambda_u * c_i$ is 0. which means
\begin{equation}
r_0-\lambda_u * c_0 < r_i - \lambda_u * c_i, \quad\forall i=1,...,n
\end{equation}
so
\begin{equation}
\frac{r_0-r_i}{c_0 - c_i} < \lambda_u, \quad\forall i=1,...,n
\end{equation}

The details of find upper bound of $\lambda$ step is given in Algorithm~\ref{UpperBoundLambda}
\begin{algorithm}[h]
\caption{Upper Bound for Lambda}
\label{UpperBoundLambda}
\begin{algorithmic}[1]
\Procedure{LambdaUpperBound}{$D_{val}$} 
\LeftComment{calculate mean reward and cost for each action}
\State action\_reward\_sum = Counter()
\State action\_cost\_sum = Counter()
\State action\_count = Counter()
\For {$(s,a,p,r,c) \in D_{val}$}
\State action\_reward\_sum[a] += r
\State action\_cost\_sum[a] += c
\State action\_count[a] += 1
\EndFor
\State action\_reward = Dict()
\State action\_cost = Dict()
\State action\_list = List()
\For {a, n $\in$ action\_count}
\State action\_list.append(a)
\State action\_reward[a] = action\_reward\_sum[a]/n
\State action\_cost[a] = action\_cost\_sum[a]/n
\EndFor
\LeftComment{search upper bound for $\lambda$}
\State $\lambda$ = 0
\State $r_0$ = action\_reward[0]
\State $c_0$ = action\_cost[0]
\For{i $\in$ range(1, len(action\_list)) }
	\State a = action\_list[i]
	\State $r_i$ = action\_reward[a]
	\State $c_i$ = action\_cost[a]
	\State value = $\frac{r_0-r_i}{c_0 - c_i}$ + 1.
	\If{value < $\lambda$}
	\State $\lambda$ = value
	\EndIf
\EndFor
\EndProcedure
\end{algorithmic}
\end{algorithm}

\end{document}